\documentclass[a4paper,UKenglish,cleveref, autoref, thm-restate]{lipics-v2021}

\usepackage{comment}
\usepackage[ruled,vlined]{algorithm2e}
\usepackage{tabularx}
\usepackage{booktabs}
\usepackage{multirow}
\usepackage{arydshln}

\usepackage{tikz}
\usetikzlibrary{trees}
\usepackage{forest}
\usetikzlibrary{shapes,arrows,positioning}

\nolinenumbers

\newcommand{\ADDAND}{\ensuremath{\text{ADD}[\land]}\xspace}
\newcommand{\OBDDAND}{\ensuremath{\text{OBDD}[\land]}\xspace}

\makeatletter
\def\thanks#1{\protected@xdef\@thanks{\@thanks
		\protect\footnotetext{#1}}}
\makeatother

\title{Scalable Precise Computation of Shannon Entropy
	  } 

\titlerunning{Scalable Precise Computation of Shannon
	Entropy} 

\author{Yong Lai}{Key Laboratory of Symbolic Computation and Knowledge Engineering  Ministry of Education, Jilin University, China}{laiy@jlu.edu.cn}{0000-0002-6882-0107}{}

\author{Haolong Tong}{College of Computer Science and Technology, Jilin University, China}{tonghl22@mails.jlu.edu.cn}{0009-0007-9320-0923}{}

\author{Zhenghang Xu}{School of Computer Science and Information Technology, Northeast Normal University, China}{zhenghangxu97@gmail.com}{0000-0001-6535-8052}{}

\author{Minghao Yin\thanks{Authors are listed alphabetically by last name. Corresponding author: Minghao Yin}}{School of Computer Science and Information Technology, Northeast Normal University, China}{ymh@nenu.edu.cn}{0000-0002-6226-2394}{}

\authorrunning{Y. Lai, H. Tong, Z. Xu, and M. Yin}  

\Copyright{Yong Lai, Haolong Tong, Zhenghang Xu and Minghao Yin} 



\category{} 

\relatedversion{} 

\supplementdetails[subcategory={Source Code}]{Software}{https://github.com/laigroup/PSE} 

\funding{This work was supported in part by Jilin Provincial Natural Science Foundation [20240101378JC], Jilin Provincial Education Department Research Project [JJKH20241286KJ], and the National Natural Science Foundation of China [U22A2098, 62172185, and 61976050].}

\acknowledgements{The authors thank the anonymous reviewers for their constructive feedback.}

\EventEditors{Jeremias Berg and Jakob Nordstr\"{o}m}
\EventNoEds{2}
\EventLongTitle{28th International Conference on Theory and Applications of Satisfiability Testing (SAT 2025)}
\EventShortTitle{SAT 2025}
\EventAcronym{SAT}
\EventYear{2025}
\EventDate{August 12--15, 2025}
\EventLocation{Glasgow, Scotland}
\EventLogo{}
\SeriesVolume{341}
\ArticleNo{4}

\begin{document}

\maketitle

\begin{abstract}
	Quantitative information flow analyses (QIF) are a class of techniques for measuring the amount of confidential information leaked by a program to its public outputs. 
	Shannon entropy is an important method to quantify the amount of leakage in QIF.
	This paper focuses on the programs modeled in Boolean constraints and optimizes the two stages of the Shannon entropy computation to implement a scalable precise tool PSE.
	In the first stage, we design a knowledge compilation language called \ADDAND that combines Algebraic Decision Diagrams and conjunctive decomposition.
	\ADDAND avoids enumerating possible outputs of a program and supports tractable entropy computation. 
	In the second stage, we optimize the model counting queries that are used to compute the probabilities of outputs. 
	We compare PSE with the state-of-the-art probabilistic approximately correct tool EntropyEstimation, which was shown to significantly outperform the previous precise tools.
	The experimental results demonstrate that PSE solved 56 more benchmarks compared to EntropyEstimation in a total of 459. For 98\% of the benchmarks that both PSE and EntropyEstimation solved, PSE is at least $10\times$ as efficient as EntropyEstimation.
\end{abstract}



\section{Introduction}
\label{sec:Intro}

Quantitative information flow (QIF) is an important approach to measuring the amount of information leaked about a secret by observing the running of a program~\cite{denning1982cryptography,gray1992toward}.
In QIF, we often quantify the leakage using entropy-theoretic notions, such as Shannon entropy~\cite{backes2009automatic,cerny2011complexity,phan2012symbolic,smith2009foundations} or
min-entropy~\cite{backes2009automatic,meng2011calculating,phan2012symbolic,smith2009foundations}.
Roughly speaking, a program in QIF can be seen as a function from a set of secret inputs $X$ to outputs $Y$ observable to an attacker who may try to infer $X$ based on the output $Y$.
Boolean formulas are a basic representation to model programs~\cite{fremont2017maximum,golia2022scalable}. 
In this paper, we focus on precisely computing the Shannon entropy of a program expressed in Boolean formulas.

Let $\varphi(X,Y)$ be a (Boolean) formula that models the relationship between the input variable set $X$ and the output variable set $Y$ in a given program, such that for any assignment of $X$, at most one assignment of $Y$ satisfies the formula $\varphi(X,Y)$.
Let $p$ represent a probability distribution defined over the set $\{ \mathit{false}, \mathit{true} \}^Y$.
For each assignment $\sigma$ to $Y$, the probability is defined as $p_{\sigma} = \frac{\left| \mathit{Sol}(\varphi(Y \mapsto \sigma)) \right|}{ \left| \mathit{Sol}(\varphi)_{\downarrow X} \right| }$, where $\mathit{Sol}(\varphi(Y \mapsto \sigma))$ denotes the set of solutions of $\varphi(Y \mapsto \sigma)$ and $\mathit{Sol}(\varphi)_{\downarrow X}$ denotes the set of solutions of $\varphi$ projected to $X$.
The Shannon entropy of $\varphi$ is $H(\varphi) = \sum_{\sigma \in 2^Y} -p_{\sigma} \log p_{\sigma} $.
Then we can immediately obtain a measure of leaked information with the computed entropy and the assumption that $X$ follows a uniform distribution~\footnote{If $X$ does not follow a uniform distribution, techniques exist for reducing the analysis to a uniform case~\cite{backes2011non}.}~\cite{klebanov2013sat}. 

The workflow of existing precise methods for computing entropy can often be divided into two stages. 
In the first stage, we enumerate possible outputs, i.e., the satisfying assignments over $Y$, while in the second stage, we compute the probability of the current output based on the number of inputs mapped to the output~\cite{golia2022scalable}.
The computation in the second stage often invokes model counting (\#SAT), which refers to computing the number of solutions $\mathit{Sol}(\varphi)$ for a given formula $\varphi$. 
Due to the exponential number of possible outputs, the current precise methods are often difficult to scale to programs with a large size of $Y$.
Therefore, researchers have increasingly focused on approximate estimation of Shannon entropy.
We remark that Golia et al.~\cite{golia2022scalable} proposed the first Shannon entropy estimation tool, EntropyEstimation, which guarantees that the estimate lies within $(1 \pm \epsilon)$-factor of $H(\varphi)$ with confidence at least $1-\delta$.
EntropyEstimation employs uniform sampling to avoid generating all outputs, and indeed scales much better than the precise methods.

As previously discussed, existing methods for precisely computing Shannon entropy struggle to scale when applied to formulas with a large set of outputs.
Theoretically, this requires performing up to $2^{|Y|}$ model counting queries.
The primary contribution of this paper is to enhance the scalability of precise Shannon entropy computation by improving both stages of the computation process.
For the first stage, we design a knowledge compilation language to guide the search and avoid exhaustive enumeration of possible outputs. 
This language augments Algebraic Decision Diagrams (ADDs), an influential representation, with conjunctive decomposition.
For the second stage, instead of performing model counting queries individually, we leverage shared component caching across successive queries.
Moreover, we exploit literal equivalence to pre-process the formula corresponding to a given program.
By integrating these techniques, we develop a Precise Shannon Entropy tool PSE.
We conducted an extensive experimental evaluation over a comprehensive set of benchmarks (459 in total) and compared PSE with the existing precise Shannon entropy computing methods and the current state-of-the-art Shannon entropy estimation tool, EntropyEstimation. 
Our experiments indicate that EntropyEstimation is able to solve 276 instances, whereas PSE surpasses this by solving an additional 56 instances.
Among the benchmarks that were solved by both PSE and EntropyEstimation, PSE is at least $10\times$ as efficient as EntropyEstimation in 98\% of these benchmarks. 

The remainder of this paper is organized as follows. 
Section \ref{sec:Notation} introduces the notation and provides essential background information.
Section \ref{sec:ADDAND} introduces Algebraic Decision Diagrams with conjunctive decomposition (\ADDAND).
Section \ref{sec:PSE} discusses the application of \ADDAND to QIF and introduces our precise entropy tool, PSE. 
Section \ref{sec:Experiments} details the experimental setup, results, and analysis.
Section \ref{sec:Related} reviews related work.
Finally, Section \ref{sec:Conclusion} concludes the paper.

\section{Notations and Background}
\label{sec:Notation}

In this paper, we focus on the programs modeled by (Boolean) formulas.
In the formulas discussed, the symbols $x$ and $y$ denote variables, and literal $l$ refers to either the variable $x$ or its negation $\neg x$, where $var(l)$ represents the variable underlying the literal $l$.
A formula $\varphi$ is constructed from the constants $\mathit{true}$, $\mathit{false}$ and variables using negation operator $\neg$, conjunction operator $\wedge$, disjunction operator $\vee$, implication operator $\rightarrow$, and equality operator $\leftrightarrow$, where $\mathit{Vars}(\varphi)$ denotes the set of variables appearing in $\varphi$.
A clause $C$ (resp. term $T$) is a set of literals representing their disjunction (resp. conjunction).
A formula in conjunctive normal form (CNF) is a set of clauses representing their conjunction.
Given a formula $\varphi$, a variable $x$, and a constant $b$, the substitution $\varphi[x \mapsto b]$ refers to the transformed formula obtained by substituting the occurrence of $x$ with $b$ throughout $\varphi$.

An assignment $\sigma$ over variable set $V$ is a mapping from $V$ to $\{ \mathit{false}, \mathit{true} \}$.
The set of all assignments over $V$ is denoted by $2^V$.
Given a subset $V' \subseteq V$, $\sigma_{\downarrow V'} = \{x \mapsto b \in \sigma \mid x \in V'\}$.
Given a formula $\varphi$, an assignment over $\mathit{Vars}(\varphi)$ satisfies $\varphi$ ($\sigma \models \varphi$) if the substitution $\varphi[\sigma]$ is equivalent to $\mathit{true}$.
Given an assignment $\sigma$, if all variables are assigned a value in $\{\mathit{false}, \mathit{true}\}$, then $\sigma$ is referred to as a complete assignment. Otherwise it is a partial assignment.
A satisfying complete assignment is also called solution or model.
We use $\mathit{Sol}(\varphi)$ to the set of solutions of $\varphi$, and model counting is the problem of computing $|\mathit{Sol}(\varphi)|$.   
Given two formulas $\varphi$ and $\psi$ over $V$, $\varphi \models \psi$ iff $\mathit{Sol}( \varphi \land \lnot\psi) = \emptyset$.

\subsection{Circuit formula and its Shannon entropy}
Given a formula $\varphi(X, Y)$ to represent the relationship between input variables $X$ and output variables $Y$, if $\sigma_{\downarrow X} = \sigma_{\downarrow X}'$ implies $\sigma = \sigma'$ for each $\sigma,\sigma' \in \mathit{Sol}(\varphi)$, then $\varphi$ is referred to as a circuit formula.
It is standard in the security community to employ circuit formulas to model programs in QIF~\cite{golia2022scalable}.

\begin{example}\label{hard-circuit-example}
	The following formula is a circuit formula with input variables $X = \{x_1, \ldots, x_{2n} \}$ and output variables $Y = \{ y_1, \ldots, y_{2n}\}$:
	$\varphi_{n}^{\mathit{sep}} = \bigwedge_{i=1}^{n}(x_i \land x_{n+i} \rightarrow y_i \land y_{n+i})    
	\wedge (\lnot x_i \lor \lnot x_{n+i} \rightarrow \lnot y_i \land \lnot y_{n+i})$.
\end{example}

    In the computation of Shannon entropy, we focus on the probability distribution of outputs.
	Let $p$ denote a probability distribution defined over the set $\{ \mathit{false}, \mathit{true} \}^Y$.
	For each assignment $\sigma$ to $Y$, i.e., $\sigma:Y \mapsto \{ \mathit{false}, \mathit{true} \}$, its weight and probability is defined as $\omega_{\sigma} = \left| \mathit{Sol}(\varphi(Y \mapsto \sigma)) \right|$ and $p_{\sigma} = \frac{\left| \mathit{Sol}(\varphi(Y \mapsto \sigma)) \right|}{ \left| \mathit{Sol}(\varphi)_{\downarrow X} \right| }$, respectively, where $\mathit{Sol}(\varphi(Y \mapsto \sigma))$ denotes the set of solutions of $\varphi(Y \mapsto \sigma)$ and $\mathit{Sol}(\varphi)_{\downarrow X}$ denotes the set of solutions of $\varphi$ projected to $X$.
	Since $\varphi$ is a circuit formula, it is easy to prove that $ \left| \mathit{Sol}(\varphi)_{\downarrow X} \right| = \left| \mathit{Sol}(\varphi)  \right|$.
	Then, the entropy of $\varphi$ is $H(\varphi) = \sum_{\sigma \in 2^Y} -p_{\sigma} \log {p_{\sigma}}$.
	Following the convention in QIF \cite {smith2009foundations}, we use base 2 for log, though the base can be chosen freely.

\subsection{Knowledge compilation}

Knowledge compilation is the approach of compiling CNF formulas into a form to support tractable reasoning tasks such as satisfiability check, equivalence check, and model counting~\cite{darwiche2002knowledge}.
Ordered binary decision diagram (OBDD)~\cite{bryant1986graph} is one of the most influential knowledge compilation forms, which supports many tractable reasoning tasks.
Each OBDD is a rooted directed acyclic graph (DAG) defined over a linear ordering of variables $\prec$. 
Each internal node $v$ is called decision node and has two outgoing edges, referred to as the low child $lo(u)$ and the high child $hi(u)$, which are typically represented by dashed and solid lines, respectively. 
Every node $u$ is labeled with a symbol $sym(u)$. 
If $u$ is a terminal node, then $sym(u) = \bot$ or $\top$, representing the Boolean constants $\mathit{false}$ and $\mathit{true}$, respectively. 
Otherwise, $sym(u)$ denotes a variable and $u$ represents $(\neg sym(u) \land \psi) \lor (sym(u) \land \psi')$, where $\psi$ and $\psi'$ are the formulas represented by $lo(u)$ and $hi(u)$, respectively.
Each decision node $v$ and its parent $u$ have $sym(u) \prec sym(v)$.
\OBDDAND~\cite{lai2017new} is an extended form of OBDD with better space efficiency.
It augments OBDD with conjunctive decomposition nodes.
Each conjunctive decomposition node $u$ has a set of children $Ch(u)$ representing formulas without shared variables, and $u$ represents a conjunction of the formulas represented by its children.
\OBDDAND also supports a set of tractable reasoning tasks, including model counting and equivalence check.

Both OBDD and \OBDDAND can only represent Boolean functions.
An Algebraic Decision Diagram (ADD)~\cite{bahar1997algebric} is an extension of OBDD to represent algebraic functions.
ADD is a compact representation of a real-valued function as a directed acyclic graph. 
While OBDD has two terminal nodes representing $\mathit{false}$ and $\mathit{true}$, ADD includes multiple terminal nodes, each assigned a real value.
The order in which decision node labels appear in all paths from the root to the terminal nodes of the ADD also align with a given ordering of variables $\prec$.
Given an assignment $\sigma$ with each variable in $\prec$, we can obtain a path in a top-down way as follows: for a decision node with $x$, we pick low child if $\sigma(x) = \mathit{false}$, and high child otherwise. ADD maps $\sigma$ to the value on the terminate node of the path. 
The original design motivation for ADD was to solve matrix multiplication, shortest path algorithms, and direct methods for numerical linear algebra~\cite{bahar1997algebric}.
In subsequent research, ADD has also been used for stochastic model checking~\cite{kwiatkowska2007stochastic}, stochastic programming~\cite{hoey2013spudd}, and weighted model counting~\cite{dudek2020addmc,lai2025pbcounter}.

\section{\ADDAND: A New Tractable Representation}
\label{sec:ADDAND}

In order to compute the Shannon entropy of a circuit formula $\varphi(X, Y)$, we need to use the probability distribution over the outputs.
Algebraic Decision Diagrams (ADDs) are an influential compact probability representation that can be exponentially smaller than the explicit representation.
Macii and Poncino~\cite{macii1996exact} showed that ADD supports efficient exact computation of entropy.
However, we observed in the experiments that the sizes of ADDs often exponentially explode with large circuit formulas.
We draw inspiration from a Boolean representation known as the Ordered Binary Decision Diagram with conjunctive decomposition (\OBDDAND)~\cite{lai2017new}, which reduces its size through recursive component decomposition and divide-and-conquer strategies. 
This approach enables the representation to be exponentially smaller than the original OBDD.
Accordingly, we propose a probabilistic representation called Algebraic Decision Diagrams with conjunctive decomposition (\ADDAND) and demonstrate it supports tractable entropy computation.
\ADDAND is a general form of ADD and is defined as follows:

\begin{definition}\label{ADDAND-definition}
	An \ADDAND is a rooted DAG, where each node $u$ is labeled with a symbol $sym(u)$.
	If $u$ is a terminal node, $sym(u)$ is a non-negative real weight, also denoted by $\omega(u)$; otherwise, $sym(u)$ is a variable (called \emph{decision} node) or operator $\wedge$ (called \emph{decomposition} node). 
	The children of a decision node $u$ are referred to as the \emph{low} child $lo(u)$ and the \emph{high} child $hi(u)$, and connected by dashed lines and solid lines, respectively, corresponding to the cases where $\mathit{var}(u)$ is assigned the value of $\mathit{false}$ and  $\mathit{true}$. 
	For a decomposition node, its sub-graphs do not share any variables.
	An \ADDAND is imposed with a linear ordering $\prec$ of variables such that given a node $u$ and its non-terminal child $v$, $\mathit{var}(u) \prec \mathit{var}(v) $.
\end{definition}

Hereafter, we denote the set of variables that appear in the graph rooted at $u$ as $\mathit{Vars}(u)$ and the set of child nodes of $u$ as $Ch(u)$.
We now turn to show how an \ADDAND defines a probability distribution:

\begin{definition}\label{def:ADDAND-weight}
	Let $u$ be an \ADDAND node over a set of variables $Y$ and let $\sigma$ be an assignment over $Y$. 
	The weight of $\sigma$ is defined as follows:
	\begin{equation*}
		\omega(\sigma,u) =  
		\begin{cases}  
			\mathit{\omega}(u) & \text{terminal} \\  
			\prod_{v \in Ch(u)}{\omega(\sigma,v)} & \text{decomposition} \\  
			\omega(\sigma,lo(u)) & \text{decision and $\sigma \models  \lnot \mathit{var}(u)$}  \\
			\omega(\sigma,hi(u)) & \text{decision and $\sigma \models  \mathit{var}(u)$}  \\
		\end{cases}
	\end{equation*}
	The weight of an non-terminal \ADDAND rooted at $u$ is denoted by $\omega(u)$ and defined as $\sum_{\sigma \in 2^{\mathit{Vars}(u)}}\omega(\sigma,u)$.
	For nodes with a non-zero weight, the probability of $\sigma$ over $u$ is defined as $p(\sigma, u) = \frac{\omega(\sigma,u)}{\omega(u)}$.
\end{definition}

Figure \ref{fig:ADDAND-Example-hard} depicts an \ADDAND representing the probability distribution of $\varphi_n^{sep}$ in Example \ref{hard-circuit-example} over its outputs with respect to $y_1 \prec y_2 \prec \cdots \prec y_{2n}$. The reader can verify that each equivalent ADD with respect to $\prec$ has an exponential number of nodes.
In the field of knowledge compilation~\cite{darwiche2002knowledge,fargier2014knowledge}, the concept of succinctness is often used to describe the space efficiency of a representation. 
Based on the following observations, we can conclude that \ADDAND is strictly more succinct than ADD. 
First, OBDD and \OBDDAND are subsets of ADD and \ADDAND, respectively. 
Second, \OBDDAND is strictly more succinct than OBDD~\cite{lai2017new}. 
Finally, each \OBDDAND cannot be transformed into a non-OBDD ADD.

\begin{figure}[h]
	\vspace{1.5cm} 
	\centering
	\resizebox{0.8\linewidth}{!} {
		\begin{forest}
			for tree={
				scale=0.6, 
				if n children=0{circle, draw, inner sep=2pt}{},
				if level=1{edge={draw, solid}}{
					if n=1{edge={draw, dashed}}{edge={draw, solid}}
				}
			}
			[$\wedge$, circle, draw, minimum width=1cm,
			label={[red,font=\footnotesize,above=0.01cm]above:$\scalebox{0.6}{$\overbrace{(-\frac{3}{4} \cdot \log\frac{3}{4} - \frac{1}{4} \cdot \log\frac{1}{4}) + \cdots + (-\frac{3}{4} \cdot \log\frac{3}{4} - \frac{1}{4} \cdot \log\frac{1}{4})}^{n}$}$},label={[blue,font=\footnotesize,left=0.01cm]left:$\scalebox{0.7}{$4^n$}$}
			[$y_1$, circle, draw, minimum width=1cm,
			label={[red,font=\footnotesize,above=0.1cm]above:$\scalebox{0.7}{$-\frac{3}{4} \cdot \log\frac{3}{4} - \frac{1}{4} \cdot \log\frac{1}{4}$}$},label={[blue,font=\footnotesize,left=0.01cm]left:$\scalebox{0.7}{$4$}$},
			name = y1
			[$y_{n+1}$, circle, draw, minimum width=1cm,
			label={[red,font=\footnotesize,right=0.1cm]right:$\scalebox{0.7}{$0$}$},label={[blue,font=\footnotesize,left=0.01cm]left:$\scalebox{0.7}{$3$}$}
			[$3$, rectangle, draw, minimum width=1cm, minimum height=1cm,
			label={[red,font=\footnotesize,below=0.1cm]below:$\scalebox{0.7}{$0$}$}]
			[$0$, rectangle, draw, minimum width=1cm, minimum height=1cm,
			label={[red,font=\footnotesize,below=0.1cm]below:$\scalebox{0.7}{$0$}$}]
			]
			[$y_{n+1}$, circle, draw, minimum width=1cm,
			label={[red,font=\footnotesize,right=0.1cm]right:$\scalebox{0.7}{$0$}$},label={[blue,font=\footnotesize,left=0.01cm]left:$\scalebox{0.7}{$1$}$}
			[$0$, rectangle, draw, minimum width=1cm, minimum height=1cm,
			label={[red,font=\footnotesize,below=0.1cm]below:$\scalebox{0.7}{$0$}$}]
			[$1$, rectangle, draw, minimum width=1cm, minimum height=1cm,
			label={[red,font=\footnotesize,below=0.1cm]below:$\scalebox{0.7}{$0$}$}]
			]
			]	
			[$y_n$, circle, draw, minimum width=1cm,
			label={[red,font=\footnotesize,above=0.1cm]above:$\scalebox{0.7}{$-\frac{3}{4} \cdot \log\frac{3}{4} - \frac{1}{4} \cdot \log\frac{1}{4}$}$},label={[blue,font=\footnotesize,left=0.01cm]left:$\scalebox{0.7}{$4$}$}, name = yn
			[$y_{2n}$, circle, draw, minimum width=1cm,
			label={[red,font=\footnotesize,right=0.1cm]right:$\scalebox{0.7}{$0$}$},label={[blue,font=\footnotesize,left=0.01cm]left:$\scalebox{0.7}{$3$}$}
			[3, rectangle, draw, minimum width=1cm, minimum height=1cm,
			label={[red,font=\footnotesize,below=0.1cm]below:$\scalebox{0.7}{$0$}$}]
			[0, rectangle, draw, minimum width=1cm, minimum height=1cm,
			label={[red,font=\footnotesize,below=0.1cm]below:$\scalebox{0.7}{$0$}$}]
			]
			[$y_{2n}$, circle, draw, minimum width=1cm,
			label={[red,font=\footnotesize,right=0.1cm]right:$\scalebox{0.7}{$0$}$},label={[blue,font=\footnotesize,left=0.01cm]left:$\scalebox{0.7}{$1$}$}
			[0, rectangle, draw, minimum width=1cm, minimum height=1cm,
			label={[red,font=\footnotesize,below=0.1cm]below:$\scalebox{0.7}{$0$}$}]
			[1, rectangle, draw, minimum width=1cm, minimum height=1cm,
			label={[red,font=\footnotesize,below=0.1cm]below:$\scalebox{0.7}{$0$}$}]
			]
			]
			]
			\begin{tikzpicture}[overlay]
				\path (y1.south) -- (yn.north) node[pos=0.5, fill=white] {$\cdots$};
			\end{tikzpicture}
		\end{forest}
	}

	\caption{An \ADDAND representing the probability distribution of $\varphi_n^{sep}$ in Example \ref{hard-circuit-example} over its outputs. 
		According to Proposition \ref{prop:omega-proposition}, the computed weight for each node is marked in blue font.
		According to Proposition \ref{prop:Entropy-proposition}, the computed entropy for each node is marked in red font.
	}
	\label{fig:ADDAND-Example-hard}
\end{figure}

\subsection{Tractable Computation of Weight and Entropy}

The computation of Shannon entropy for an \ADDAND relies on its weight.
We first demonstrate that, for an \ADDAND node $u$, its weight $\mathit{\omega}(u)$ can be computed in polynomial time.
\begin{proposition}\label{prop:omega-proposition}
	Given a non-terminal node $u$ in \ADDAND, its weight $\mathit{\omega}(u)$ can be recursively computed as follows in polynomial time:
	\begin{equation*}
		\mathit{\omega}(u) =  
		\begin{cases}  
			\prod_{v \in Ch(u)}{\mathit{\omega}(v)}  & \text{decomposition}   \\
			2^{n_0} \cdot \mathit{\omega}(lo(u)) + 2^{n_1} \cdot \mathit{\omega}(hi(u))  & \text{decision}
			
		\end{cases}
	\end{equation*}
	where $n_0 = |\mathit{Vars}(u)| - |\mathit{Vars}(lo(u))| - 1 $ and $n_1 = |\mathit{Vars}(u)| - |\mathit{Vars}(hi(u))| - 1$.
	
	\begin{proof}
		The time complexity is immediate by using dynamic programming.
		We prove the equation can compute the weight correctly by induction on the number of variables of the \ADDAND rooted at $u$.
		It is obvious that the weight of a terminal node is the real value labeled. 
		For the case of the $\wedge$ node, since the variables of the child nodes are all disjoint, it can be easily seen from Definition \ref{def:ADDAND-weight}.
		Next, we will prove the case of the decision node.
		Assume that when $|\mathit{Vars}(u)| \le n$, this proposition holds. 
		For the case where $|\mathit{Vars}(u)| = n + 1$, we use $Y_0$ and $Y_1$ to denote $\mathit{Vars}(lo(u))$ and $\mathit{Vars}(hi(u))$, and we have $|Y_0| \le n$ and $|Y_1| \le n$.
		Thus, $\mathit{\omega}(lo(u))$ and $\mathit{\omega}(hi(u))$ can be computed correctly.
		According to Definition \ref{def:ADDAND-weight}, $w(u) = \sum_{\sigma \in 2^{\mathit{Vars}(u)}}\omega(\sigma,u)$.
		The assignments over $\mathit{Vars}(u)$ can be divided into two categories: 
		\begin{itemize}
			\item The assignment $\sigma \models \lnot \mathit{var}(u)$: 
			It is obvious that $\omega(\sigma, u) = \omega(\sigma_{\downarrow Y_0}, lo(u))$. 
			Each assignment over $Y_0$ can be extended to exactly $2^{n_0}$ different assignments over $\mathit{Vars}(u)$ in this category. Thus, we have the following equation:
			$$\sum_{\sigma \in 2^{\mathit{Vars}(u)} \land \sigma \models \lnot \mathit{var}(u)}\omega(\sigma, u) = 2^{n_0} \cdot \mathit{\omega}(lo(u)).$$
			\item The assignment $\sigma \models \mathit{var}(u)$: This case is similar to the above case.
		\end{itemize}
		To sum up, we can obtain that $\mathit{\omega}(u) = 2^{n_0} \cdot \mathit{\omega}(lo(u)) + 2^{n_1} \cdot \mathit{\omega}(hi(u))$.	
	\end{proof}
	
\end{proposition}

We now present how \ADDAND computes Shannon entropy in polynomial time.
\begin{proposition}\label{prop:Entropy-proposition}
	Given an \ADDAND rooted at $u$, if $\omega(u) = 0$, we define its entropy $\mathit{H}(u)$ as $0$, and otherwise its entropy can be recursively computed in polynomial time as follows:
	\begin{equation*}
		\mathit{H}(u) =  
		\begin{cases}  
			0 & \text{terminal}   \\
			\sum_{v \in Ch(u)}{\mathit{H}(v)}  & \text{decomposition}   \\
			\begin{gathered}	
				p_0 \cdot (H(lo(u)) + n_0 - \log p_0) + p_1 \cdot (H(hi(u)) + n_1 - \log p_1)
			\end{gathered} & \text{decision}
			
		\end{cases}
	\end{equation*}
	where  $n_0 = |\mathit{Vars}(u)| - |\mathit{Vars}(lo(u))| - 1 $, $n_1 = |\mathit{Vars}(u)| - |\mathit{Vars}(hi(u))| - 1$, $p_{0} = \frac{2^{n_0} \cdot \mathit{\omega}(lo(u))}{\omega(u)}$, and $p_{1} =  \frac{ 2^{n_1} \cdot \mathit{\omega}(hi(u))}{\omega(u)}$.

	\begin{proof}
		According to Proposition \ref{prop:omega-proposition}, $\omega(u)$ can be computed in polynomial time, and therefore the time complexity in this proposition is obvious. 
		Next we prove the correctness of the computation method.
		The case of terminal nodes is obviously correct. The case of decomposition follows directly from the additivity property of entropy. Next, we show the correctness of the case of decision.
		
		Let $H_0(u)$ be $-\sum_{\sigma \models \lnot sym(u)} p(u,\sigma) \log p(u, \sigma)$ and $H_1(u)$ be $-\sum_{\sigma \models sym(u)} p(u,\sigma) \log p(u,\sigma)$. Similar to proposition \ref{prop:omega-proposition}, we can obtain $H(u) = 2^{n_0} \cdot H_0(u) + 2^{n_1} \cdot H_1(u)$. 
		The assignments over $\mathit{Vars}(u)$ can be divided into two categories: 
		\begin{itemize}
			\item The assignment $\sigma \models \lnot \mathit{var}(u)$: 
			According to definition \ref{def:ADDAND-weight}, the probability $p(\sigma,\mathit{\omega})$ satisfies $p(\sigma,\mathit{\omega}) = \frac{\mathit{\omega}(\sigma,lo(u))}{\mathit{\omega}(u)}$.
			Given $p_{0} = \frac{2^{n_0} \cdot \mathit{\omega}(lo(u))}{\omega(u)}$, it follows that $\frac{\mathit{\omega}(lo(u))}{\mathit{\omega}(u)} = \frac{p_0}{2^{n_0}}$.
			Substituting this into the expression for $p(\sigma,u)$, we derive $p(\sigma,u) =\frac{\mathit{\omega}(\sigma,lo(u))}{\mathit{\omega}(lo(u)} \cdot p_0 \cdot 2^{-n_0} = p(\sigma,lo(u)) \cdot p_0 \cdot 2^{-n_0}$.
			$H_0(u)$ then expands as $H_0(u) = -\sum_{\sigma \models \lnot sym(u)} p(\sigma,u) \log p(\sigma,u) = -\sum_{\sigma \models \lnot sym(u)} p(\sigma,lo(u)) \cdot p_0 \cdot 2^{-n_0} \cdot (\log p(\sigma,lo(u)) + \log p_0 - n_0) = p_0 \cdot 2^{-n_0} \cdot [-\log p_0 \cdot \sum_{\sigma \models \lnot sym(u)} p(\sigma,lo(u)) + n_0 \cdot \sum_{\sigma \models \lnot sym(u)} p(\sigma,lo(u)) - \sum_{\sigma \models \lnot sym(u)} p(\sigma,lo(u)) \cdot  \log p(\sigma,lo(u))]$.
			Noting that $\sum_{\sigma \models \lnot sym(u)} p(\sigma,lo(u)) = 1, - \sum_{\sigma \models \lnot sym(u)} p(\sigma,lo(u)) \cdot  \log p(\sigma,lo(u)) = H(lo(u))$,
			we simplify $H_0(u) = p_0 \cdot 2^{-n_0} \cdot (-\log p_0 + n_0 + H(lo(u)))$.
			\item The assignment $\sigma \models \mathit{var}(u)$: This case is similar to the above case. It is easy to obtain $H_1(u) = p_1 \cdot 2^{-n_1} \cdot (-\log p_1 + n_1 + H(hi(u)))$.
		\end{itemize}
		To sum up, we can obtain that $H(u) = p_0 \cdot (H(lo(u)) + n_0 - \log p_0) + p_1 \cdot (H(hi(u)) + n_1 - \log p_1)$
	\end{proof}
	
\end{proposition}

We conclude this section by explaining why ordering is used in the design of \ADDAND.
In fact, Propositions \ref{prop:omega-proposition}--\ref{prop:Entropy-proposition} remain valid even when we use only the more general read-once property, where each variable appears at most once along any path from the root of an \ADDAND to a terminal node.
First, our experimental results indicate that the linear ordering determined by the minfill algorithm in our tool PSE outperforms the dynamic orderings employed in the state-of-the-art model counters, where the former imposes the orderedness and the latter imposes the read-once property.
Second, \ADDAND can provide tractable equivalence checking between probability distributions beyond this study.

\section{PSE: Scalable Precise Entropy Computation}
\label{sec:PSE}

In this section, we introduce our tool PSE, designed to compute the Shannon entropy of a given circuit CNF formula with respect to its output variables.
PSE, as presented in Algorithm \ref{PSE}, takes as input a CNF formula $\varphi$, an input set $X$, and an output set $Y$, and returns the Shannon entropy $\mathit{H}(\varphi)$ of the formula.
Like other tools for computing Shannon entropy, PSE follows a two-stage process: the $Y$-stage (corresponding to outputs) and the $X$-stage (corresponding to inputs).
In the $X$-stage (lines \ref{PSE-line:count}--\ref{PSE-line:count-return}), we perform multiple optimized model counting operations on sub-formulas over variables in $X$, where the leaves of \ADDAND are implicitly generated. 
The optimization technique is discussed in Section \ref{sec:implementation}.
In the $Y$-stage (the remaining lines), we conduct a search within the \ADDAND framework to precisely compute the Shannon entropy, where the internal nodes of \ADDAND are implicitly generated.
The following observation states the input of each recursive call is still a circuit formula and the two input formulas of a call corresponding to a decision node in \ADDAND have the same output variables.

\begin{observation}\label{prop:circuit-proposition}
	Given a circuit formula $\varphi(X, Y)$ and a partial assignment $\sigma$ without any input variables, we have the following properties:
	\begin{itemize}
		\item $\varphi[\sigma](X, Y \setminus \mathit{Vars}(\sigma))$ is a circuit formula;
		\item Each $\psi_i(X, Y \cap \mathit{Vars}(\psi_i))$ is a circuit formula if $\varphi = \bigwedge_{i = 1}^m\psi_i$ and for  
		$1 \le i \ne j  \le m$, $\mathit{Vars}(\psi_i) \cap \mathit{Vars}(\psi_j) = \emptyset$;
		\item If $\varphi[\sigma] \equiv \mathit{true}$, $\sigma$  contains each output variable.
	\end{itemize}
	
\end{observation}

\begin{proof}
The first two properties obviously hold when $\varphi$ is unsatisfiable. Thereby, we assume $\varphi$ is satisfiable. 
For the first property, let $\varphi'$ be $\varphi[\sigma](X, Y \setminus \mathit{Vars}(\sigma))$.
For each $\sigma',\sigma'' \in \mathit{Sol}(\varphi)$, $\sigma'_{\downarrow X} = \sigma''_{\downarrow X}$ implies $\sigma' = \sigma''$.
$Sol(\varphi')$ can be seen as a subset of $Sol(\varphi)$.
Consequently, for each $\sigma',\sigma'' \in \mathit{Sol}(\varphi')$, $\sigma'_{\downarrow X} = \sigma''_{\downarrow X}$ still implies $\sigma' = \sigma''$, concluding that $\varphi'$ is also a circuit formula. 

For the second property, each solution of $\psi_i$ can be obtain from a solution of $\varphi$.
Let $\sigma',\sigma''$ be two solutions of $\varphi$.
We only need to prove that $\sigma'_{\downarrow {X \cap \mathit{Vars}(\psi_i)} } = \sigma''_{\downarrow {X \cap \mathit{Vars}(\psi_i)} }$ implies $\sigma'_{\downarrow {(X \cup Y) \cap \mathit{Vars}(\psi_i)} } = \sigma''_{\downarrow {(X \cup Y) \cap \mathit{Vars}(\psi_i)}}$.
We construct another solution of $\varphi$, $\sigma''' = \sigma''_{\downarrow {(X \cup Y) \cap \mathit{Vars}(\psi_i)}} \cup \sigma'_{\downarrow {(X \cup Y) \setminus \mathit{Vars}(\psi_i)}}$.
Then $\sigma' = \sigma'''$, which implies $\sigma'_{\downarrow {(X \cup Y) \cap \mathit{Vars}(\psi_i)} } = \sigma'''_{\downarrow {(X \cup Y) \cap \mathit{Vars}(\psi_i)}} = \sigma''_{\downarrow {(X \cup Y) \cap \mathit{Vars}(\psi_i)}}$.

We prove the last property by contradiction. 
Suppose that $\varphi[\sigma] \equiv \mathit{true}$, and $\sigma$ is a partial assignment with only one free variable $y \in Y$. 
Then the value of $y$ can take either $\mathit{false}$ or $\mathit{true}$.
That is, $\sigma \cup \{y = \mathit{false}\}$ and $\sigma \cup \{y = \mathit{true}\}$
are solutions of $\varphi$, which contradicts the definition of circuit formula. 
\end{proof}

\begin{algorithm}[h]
	\caption{PSE($\varphi$,$X$,$Y$)}
	\label{PSE}
	\LinesNumbered
    \DontPrintSemicolon
	\KwIn{A circuit CNF formula $\varphi$ with input variables $X$ and output variables $Y$}
	\KwOut{the entropy of $\varphi$}
	\lIf{$\mathit{Cache_H}(\varphi) \neq nil $} {\Return $\mathit{Cache_H}(\varphi)$}
	
	\If{$ Y = \emptyset $}
	{
		$\mathit{Cache_\#}(\varphi) \leftarrow $ \texttt{CountModels}($\varphi$) \label{PSE-line:count}
		
		\Return  $\mathit{Cache_H}(\varphi)  \leftarrow 0$	\label{PSE-line:count-return}
	}
	
	$ \Psi = \texttt{Decompose}(\varphi) $\label{PSE-line:Decompose}
	
	\If{$ |\Psi| > 1 $} 
	{
		$\mathit{Cache_H}(\varphi) \leftarrow \sum_{\psi \in \Psi}{\mathbf{PSE}(\psi,X,Y \cap \mathit{Vars}(\psi_i))}$
		
		$Cache_\#(\varphi) \leftarrow \prod_{\psi \in \Psi}{Cache_\#(\psi)}$
		
		\Return $\mathit{Cache_H}(\varphi) $
	}
	
	$ y \leftarrow \texttt{PickGoodVar}(Y)$
	
	$ \varphi_0 \leftarrow \varphi[y \mapsto \mathit{false}]$; $ \varphi_1 \leftarrow  \varphi[y \mapsto \mathit{true}]$
	
	$\mathit{H}_0 \leftarrow \mathbf{PSE}$($\varphi_0,X,Y \backslash \{y\}$)\;
	$\mathit{H}_1 \leftarrow  \mathbf{PSE}$($\varphi_1,X,Y \backslash \{y\}$) 
	
	$\mathit{Cache}_\#(\varphi) \leftarrow \mathit{Cache}_\#(\varphi_0)  +  \mathit{Cache}_\#(\varphi_1)$

	$p_0 = \frac{\mathit{Cache}_\#(\varphi_0)}{\mathit{Cache}_\#(\varphi)}$; $p_1 = \frac{\mathit{Cache}_\#(\varphi_1)}{\mathit{Cache}_\#(\varphi)}$

	$\mathit{H} \leftarrow p_0 \cdot (H_0 - \log p_0) + p_1 \cdot (H_1 - \log p_1)$
	
	\Return $\mathit{Cache_H}(\varphi)  \leftarrow \mathit{H}$
	
\end{algorithm}

In line 1, if the formula $\varphi$ is cached, its corresponding entropy is returned.
If the current set $Y$ is empty (in line 2), this indicates that a satisfiable assignment has been found under the restriction of the output set $Y$.
We do not explicitly handle the case where $\varphi$ evaluates to $\mathit{true}$, as this naturally implies that $Y$ is empty, as indicated by Observation \ref{prop:circuit-proposition}.
Consequently, the scenario in which the set $Y$ is empty inherently encompasses the case where $\varphi$ evaluates to $\mathit{true}$.
Lines 3–4 perform model counting on the residual formula and compute its entropy $\mathit{H}$, corresponding to the terminal case of Proposition \ref{prop:Entropy-proposition}.
We invoke the \texttt{Decompose} function in line 5 to determine whether the formula $\varphi$ can be decomposed into multiple components.
In lines 6–9, if $\varphi$ can be decomposed into multiple sub-components, we compute the model count and entropy of each component $\psi$, and subsequently derive the entropy of the formula $\varphi$. 
In this case, computing the model count and computing the entropy correspond respectively to the $\wedge$ cases in Propositions \ref{prop:omega-proposition} and \ref{prop:Entropy-proposition}.
When there is only one component, we select a variable from $Y$ in line 10.
The \texttt{PickGoodVar} function operates as a heuristic algorithm designed to select a variable from set $Y$, with the selection criteria determined by the specific heuristic employed.
Moving forward, line 11 generates the residual formulas $\varphi_0$ and  $\varphi_1$, corresponding to assigning the variable $y$ to $\mathit{false}$ and $\mathit{true}$, respectively.
Subsequently, lines 12 and 13 recursively compute the entropy $\mathit{H}$ for each derived formula.
Since $\varphi$ is a circuit formula, all residual formulas generated in the recursive process after making decisions on variables in $Y$ remain circuit formulas.
It follows from Observation \ref{prop:circuit-proposition} that when computing the Shannon entropy of the circuit formula, $n_0 = n_1 = 0$.
The model count of $\varphi$ is cached in line 14, corresponding to the decision node case in Proposition \ref{prop:omega-proposition}.
Finally, in lines 15–16, we compute the entropy of  $\varphi$ (corresponding to the third case in Proposition \ref{prop:Entropy-proposition}), store it in the cache, and return it as the result in line 17.

\begin{example}\label{easy-circuit-example}
	Consider the following circuit CNF formula with input variables $X = \{x_1, x_2, x_3, x_4 ,x_5 \}$ and output variables $Y = \{ y_1, y_2, y_3, y_4 , y_5\}$:
	\begin{equation*}
		\begin{split}
			\varphi(X, Y) = & (x_2 \vee x_3 \vee y_3) \wedge (\neg y_3 \vee \neg y_4) \wedge 
			(x_2 \vee y_3) \wedge (\neg x_2 \vee y_4) 
			\wedge (\neg x_1 \vee \neg y_1) \wedge (x_1 \vee y_1) \wedge 
			\\
			&  (\neg x_4 \vee x_5 \vee y_2) \wedge 
			(x_4 \vee \neg x_5 \vee y_2) \wedge
			(\neg x_4 \vee \neg x_5 \vee \neg y_2) \wedge
			(x_4 \vee x_5 \vee y_2) \wedge 
			\\ 
			& (\neg y_1 \vee \neg y_5) \wedge
			(y_1 \vee y_5) \wedge (y_1 \vee x_4 \vee x_5) \wedge (y_1 \vee y_3 \vee y_4)
		\end{split} 
	\end{equation*}
Figure \ref{fig:ADDAND-Example-easy} illustrates the execution trace of PSE taking in $\varphi$ with the variable ordering $y_1 \prec y_2 \prec y_3 \prec y_4 \prec y_5$, which is an implicit \ADDAND.
If we do not perform decomposition in line \ref{PSE-line:Decompose}, the search trace is depicted in Figure \ref{fig:ADD-Example-easy}, an ADD structure.
It is evident that \ADDAND and ADD yield consistent results, both in terms of Shannon entropy computation and model counting. 
After merging identical terminal nodes, the \ADDAND contains 14 nodes, which is fewer than the 24 nodes in the ADD.
A comparison between Figure \ref{fig:ADDAND-Example-easy} and Figure \ref{fig:ADD-Example-easy} demonstrates the succinctness of the \ADDAND structure.

\begin{figure}[h]
	\resizebox{\linewidth}{!} {
	\begin{tikzpicture}[
		node/.style={
			circle,
			draw=black, 
			solid, 
			minimum width=1cm,
			black
		},
		leaf/.style={
			rectangle,
			draw=black, 
			solid, 
			minimum size=1cm,
		},
		edge_solid/.style={
			solid,
			thick
		},
		edge_dashed/.style={
			dashed,
			thick
		},
		entropy_label/.style={
			red,
			align=center
		},
		weight_label/.style={
			blue,
			align=center
		},
		level distance=3cm,
		level 1/.style={sibling distance=8cm},
		level 2/.style={sibling distance=4cm},
		level 3/.style={sibling distance=2.3cm},
		level 4/.style={sibling distance=1.2cm},
		level 5/.style={sibling distance=1.2cm}
		]
		
		\node[node,label={[entropy_label]above:{$ \frac{3}{7}\cdot(-\frac{1}{3}\log\frac{1}{3} - \frac{2}{3}\log\frac{2}{3} -\log\frac{1}{2}) + \frac{4}{7}\cdot(-\log\frac{1}{2} -\frac{1}{4}\log\frac{1}{4} - \frac{3}{4}\log\frac{3}{4}) -\frac{3}{7}\log\frac{3}{7} - \frac{4}{7}\log\frac{4}{7}$}},label={[weight_label]right:{$28$}}] (y1) {$y_1$}
		child[edge_dashed] {
			node[node,label={[entropy_label]above:{$-\frac{1}{3}\log\frac{1}{3} - \frac{2}{3}\log\frac{2}{3} -\log\frac{1}{2}$}},label={[weight_label]right:{$12$}}] (wedge1) {$\wedge$}
			child[edge_solid] {
				node[node,label={[entropy_label]above:{$-\frac{1}{3}\log\frac{1}{3} - \frac{2}{3}\log\frac{2}{3}$}},label={[weight_label]right:{$3$}}] (y4a) {$y_2$}
				child[edge_dashed] { node[leaf] (leaf1) {1} }
				child[edge_solid] { node[leaf] (leaf2) {2} }
			}
			child[edge_solid] {
				node[node,label={[entropy_label]below:{$0$}},label={[weight_label]right:{$1$}}] (y4a) {$y_5$}
				child[edge_dashed] { node[leaf] (leaf3) {0} }
				child[edge_solid] { node[leaf] (leaf4) {1} }
			}
			child[edge_solid] {
				node[node,label={[entropy_label]above:{$-\frac{1}{2}\log\frac{1}{2} - \frac{1}{2}\log\frac{1}{2}$}},label={[weight_label]right:{$4$}}] (y2) {$y_3$}
				child[edge_dashed] {
					node[node,label={[entropy_label]above:{$0$}},label={[weight_label]right:{$2$}}] (y3a) {$y_4$}
					child[edge_dashed] { node[leaf] (leaf5) {0} }
					child[edge_solid] { node[leaf] (leaf6) {2} }
				}
				child {
					node[node,label={[entropy_label]above:{$0$}},label={[weight_label]left:{$2$}}] (y3b) {$y_4$}
					child[edge_dashed] { node[leaf] (leaf7) {2} }
					child[edge_solid] { node[leaf] (leaf8) {0} }
				}
			}
		}
		child {
			node[node,label={[entropy_label]above:{$-\log\frac{1}{2} -\frac{1}{4}\log\frac{1}{4} - \frac{3}{4}\log\frac{3}{4}$}},label={[weight_label]right:{$16$}}] (wedge2) {$\wedge$}
			child[edge_solid] { 
				node[node,edge_solid] at (y2) {}
			}
			child {
				node[node,label={[entropy_label]above:{$-\frac{1}{4}\log\frac{1}{4} - \frac{3}{4}\log\frac{3}{4}$}},label={[weight_label]right:{$4$}}] (y4b) {$y_2$}
				child[edge_dashed] { node[leaf] (leaf9) {1} }
				child[edge_solid] { node[leaf] (leaf10) {3} }
			}
			child[edge_solid] {
				node[node,label={[entropy_label]above:{$0$}},label={[weight_label]right:{$1$}}] (y4a) {$y_5$}
				child[edge_dashed] { node[leaf] (leaf11) {1} }
				child[edge_solid] { node[leaf] (leaf12) {0} }
			}
		};
		
		\foreach \n in {leaf1,leaf2,leaf3,leaf4,leaf5,leaf6,leaf7,leaf8,leaf9,leaf10,leaf11,leaf12} {
			\node[red, below=1mm] at (\n.south) {0};
		}
	\end{tikzpicture}    
	}
	\caption{The execution example of PSE on Example \ref{easy-circuit-example} follows the variable order of $y_1 \prec y_2 \prec y_3 \prec y_4 \prec y_5$, where the corresponding computational trajectory is represented as an \ADDAND. The entropy computation process performed by PSE is explicitly annotated in red font. 
	The calculation process of weight (number of models) is presented in blue font.
	}
	\label{fig:ADDAND-Example-easy}
\end{figure}

\begin{figure}[h]
	\resizebox{\linewidth}{!} {
		\begin{tikzpicture}[
			node/.style={
				circle,
				draw=black, 
				solid, 
				minimum width=1cm,
				black
			},
			leaf/.style={
				rectangle,
				draw=black, 
				solid, 
				minimum size=1cm,
			},
			edge_solid/.style={
				solid,
				thick
			},
			edge_dashed/.style={
				dashed,
				thick
			},
			weight_label/.style={
				blue,
				align=center
			},
			entropy_label/.style={
				red,
				align=center
			},
			level distance=3cm,
			level 1/.style={sibling distance=8cm},
			level 2/.style={sibling distance=4cm},
			level 3/.style={sibling distance=2.2cm},
			level 4/.style={sibling distance=3.2cm},
			level 5/.style={sibling distance=1.2cm},
			]
			
			\node[node,label={[entropy_label]above:{$ \frac{3}{7}\cdot(\frac{1}{3} + \frac{2}{3} -\frac{1}{3}\log\frac{1}{3} - \frac{2}{3}\log\frac{2}{3}) + \frac{4}{7}\cdot(\frac{1}{4} + \frac{3}{4} -\frac{1}{4}\log\frac{1}{4} - \frac{3}{4}\log\frac{3}{4}) -\frac{3}{7}\log\frac{3}{7} - \frac{4}{7}\log\frac{4}{7}$}},label={[weight_label]right:{$28$}}] (y1) {$y_1$}
			child[edge_dashed] {
				node[node,label={[entropy_label]above:{$\frac{1}{3} + \frac{2}{3} -\frac{1}{3}\log\frac{1}{3} - \frac{2}{3}\log\frac{2}{3}$}},label={[weight_label]right:{$12$}}] (y4a) {$y_2$}
				child[edge_dashed] {
					node[node,label={[entropy_label]above:{$-\frac{1}{2}\log\frac{1}{2} - \frac{1}{2}\log\frac{1}{2}$}},label={[weight_label]right:{$4$}}] (y2a) {$y_3$}
					child[edge_dashed] {
						node[node,label={[entropy_label]above:{$0$}},label={[weight_label]right:{$2$}}] (y3a) {$y_4$}
						child[edge_solid] {
							node[node,label={[entropy_label]above:{$0$}},label={[weight_label]below:{$2$}}] (y5b) {$y_5$}
							child[edge_dashed] { node[leaf] (leaf1) {0} }
							child[edge_solid] { node[leaf] (leaf2) {2} }
						}
						child[edge_dashed] {
							node[node,label={[entropy_label]above:{$0$}},label={[weight_label]below:{$0$}}] (y5a) {$y_5$}
							child[edge_dashed] { node[leaf] (leaf3) {0} }
							child[edge_solid] { node[leaf] (leaf4) {0} }
						}
					}
					child[edge_solid] {
						node[node,label={[entropy_label]above:{$0$}},label={[weight_label]right:{$2$}}] (y3b) {$y_4$}
						child[edge_dashed] {
							node[node,edge_solid] at (y5b) {}
						}
						child {
							node[node,edge_solid] at (y5a) {}
						}
					}
				}
				child[edge_solid] {
					node[node,label={[entropy_label]above:{$-\frac{1}{2}\log\frac{1}{2} - \frac{1}{2}\log\frac{1}{2}$}},label={[weight_label]right:{$8$}}] (y2) {$y_3$}
					child[edge_dashed] {
						node[node,label={[entropy_label]above:{$0$}},label={[weight_label]right:{$4$}}] (y3c) {$y_4$}
						child[edge_dashed] { 
							node[node,edge_solid] at (y5a) {}
						}
						child[edge_solid] { 
							node[node,label={[entropy_label]above:{$0$}},label={[weight_label]below:{$4$}}] (y5c) {$y_5$}
							child[edge_dashed] { node[leaf] (leaf5) {0} }
							child[edge_solid] { node[leaf] (leaf6) {4} }
						}
					}
					child {
						node[node,label={[entropy_label]above:{$0$}},label={[weight_label]right:{$4$}}] (y3d) {$y_4$}
						child[edge_dashed] { 
							node[node,edge_solid] at (y5c) {}
						}
						child[edge_solid] { 
							node[node,edge_solid] at (y5a) {}
						}
					}
				}
			}
			child {
				node[node,label={[entropy_label]above:{$\frac{1}{4} + \frac{3}{4} -\frac{1}{4}\log\frac{1}{4} - \frac{3}{4}\log\frac{3}{4}$}},label={[weight_label]right:{$16$}}] (y4b) {$y_2$}
				child[edge_dashed] {
					node[node,label={[entropy_label]above:{$-\frac{1}{2}\log\frac{1}{2} - \frac{1}{2}\log\frac{1}{2}$}},label={[weight_label]right:{$4$}}] (y2c) {$y_3$}
					child[edge_dashed] {
						node[node,label={[entropy_label]above:{$0$}},label={[weight_label]right:{$2$}}] (y3e) {$y_4$}
						child[edge_dashed] {
							node[node,edge_solid] at (y5a) {}
						}
						child[edge_solid] {
							node[node,label={[entropy_label]above:{$0$}},label={[weight_label]below:{$2$}}] (y5d) {$y_5$}
							child[edge_dashed] { node[leaf] (leaf7) {2} }
							child[edge_solid] { node[leaf] (leaf8) {0} }
						}
					}
					child[edge_solid] {
						node[node,label={[entropy_label]above:{$0$}},label={[weight_label]right:{$2$}}] (y3f) {$y_4$}
						child[edge_dashed] {
							node[node,edge_solid] at (y5d) {}
						}
						child {
							node[node,edge_solid] at (y5a) {}
						}
					}
				}
				child[edge_solid] {
					node[node,label={[entropy_label]above:{$-\frac{1}{2}\log\frac{1}{2} - \frac{1}{2}\log\frac{1}{2}$}},label={[weight_label]right:{$12$}}] (y2) {$y_3$}
					child[edge_dashed] {
						node[node,label={[entropy_label]above:{$0$}},label={[weight_label]right:{$6$}}] (y3g) {$y_4$}
						child[edge_dashed] { 
							node[node,edge_solid] at (y5a) {}
						}
						child[edge_solid] { 
							node[node,label={[entropy_label]above:{$0$}},label={[weight_label]below:{$6$}}] (y5e) {$y_5$}
							child[edge_dashed] { node[leaf] (leaf9) {6} }
							child[edge_solid] { node[leaf] (leaf10) {0} }
						}
					}
					child {
						node[node,label={[entropy_label]above:{$0$}},label={[weight_label]right:{$6$}}] (y3d) {$y_4$}
						child[edge_dashed] { 
							node[node,edge_solid] at (y5e) {}
						}
						child[edge_solid] { 
							node[node,edge_solid] at (y5a) {}
						}
					}
				}
			};
			\foreach \n in {leaf1,leaf2,leaf3,leaf4,leaf5,leaf6,leaf7,leaf8,leaf9,leaf10} {
				\node[red, below=1mm] at (\n.south) {0};
			}
		\end{tikzpicture}  
	}
	\caption{An ADD structure, constructed in Example \ref{easy-circuit-example}, follows the variable order of $y_1 \prec y_2 \prec y_3 \prec y_4 \prec y_5$. According to Proposition \ref{prop:omega-proposition}, weight is marked in blue font, and according to Proposition \ref{prop:Entropy-proposition}, entropy is marked in red font.
	}
	\label{fig:ADD-Example-easy}
\end{figure}
\end{example}

From the aforementioned example, we observe that the search space of PSE corresponds to an \ADDAND that represents the weights of assignments for the output variables.
Thus, Propositions \ref{prop:omega-proposition}--\ref{prop:Entropy-proposition} and Observation \ref{prop:circuit-proposition} ensure that the entropy of the original formula is obtained from the root call of PSE.

\subsection{Implementation}
\label{sec:implementation}

We now discuss the implementation details that are crucial for the runtime efficiency of PSE. 
Specifically, leveraging the tight interplay between entropy computation and model counting, our methodology integrates a variety of state-of-the-art techniques in model counting.

In the $X$-stage of algorithm \ref{PSE}, we have the option to employ various methodologies for the model counting query denoted by \texttt{CountModels} in line 3.
The first method involves individually employing state-of-the-art model counters, such as SharpSAT-TD~\cite{korhonen2021integrating}, Ganak~\cite{sharma2019ganak}, and ExactMC~\cite{lai2021power}.
The second method, known as \textsf{ConditionedCounting}, requires the preliminary construction of a representation for the original formula $\varphi$ to support linear model counting. 
The knowledge compilation languages that can be used for this method include d-DNNF~\cite{darwiche2004new}, \OBDDAND~\cite{lai2017new}, and SDD~\cite{choi2013compiling}.
Upon reaching line 3, the algorithm executes conditioned model counting, utilizing the compiled representation of the formula and incorporating the partial assignment derived from the ancestor calls.
The last method, \textsf{SharedCounting}, also relies on exact model counters but, unlike the first method, it shares the component cache across all model counting queries using a strategy called \textsf{XCache}. 
To distinguish it from the caching approach used in the $X$-stage, the caching method in the $Y$-stage is referred to as \textsf{YCache}.
Our experimental observations indicate that the \textsf{SharedCounting} method is the most effective within the PSE framework.

\textbf{Conjunctive Decomposition}
We employed dynamic component decomposition (well-known in model counting and knowledge compilation) to divide a formula into components, thereby enabling the dynamic programming calculation of their corresponding entropy, as stated in Proposition \ref{prop:Entropy-proposition}.

\textbf{Variable Decision Heuristic} 
We implemented the current state-of-the-art model counting heuristics for picking variables from $Y$ in the computation of Shannon entropy, including \textsf{VSADS}~\cite{sang2005heuristics}, \textsf{minfill}~\cite{darwiche2009modeling}, the \textsf{SharpSAT-TD} heuristic~\cite{korhonen2021integrating}, and \textsf{DLCP}~\cite{lai2021power}.
Our experiments consistently demonstrate that the \textsf{minfill} heuristic exhibits the best performance. Therefore, we adopt the \textsf{minfill} heuristic as the default option for our subsequent experiments.

\textbf{Pre-processing} 
We have enhanced our entropy tool, PSE, by incorporating an advanced pre-processing technique that capitalizes on literal equivalence in model counting. 
This idea is inspired by the work of Lai et al.~\cite{lai2021power} on capturing literal equivalence in model counting.
Initially, we extract equivalent literals to simplify the formula. Subsequently, we restore the literals associated with the variables in set $Y$ to prevent the entropy of the formula from becoming non-equivalent after substitution. 
This targeted restoration is sufficient to ensure the equivalence of entropy calculations.
The new pre-processing method is called \textsf{Pre} in the following.
This pre-processing approach is motivated by two primary considerations. 
Firstly, preprocessing based on literal equivalence can simplify the formula and enhance the efficiency of subsequent model counting. 
Secondly, and more crucially, it can reduce the treewidth of tree decomposition, which is highly beneficial for the variable heuristic method based on tree decomposition and contributes to improving the solving efficiency.

\section{Experiments}
\label{sec:Experiments}

We implemented a prototype of PSE in C++ and performed evaluations in order to understand its performance.
We experimented on benchmarks from the same domains as the state-of-the-art Shannon entropy tool EntropyEstimation~\cite{golia2022scalable}, that is, QIF benchmarks, plan recognition, bit-blasted versions of SMTLIB benchmarks, QBFEval competitions, program synthesis, and combinatorial circuits~\cite{lee2018solving} \footnote{The paper of EntropyEstimation~\cite{golia2022scalable} does not mention the domains of program synthesis and combinatorial circuits but actually presents benchmarks in these two domains.}.
EntropyEstimation reported results only for 96 successfully solved benchmarks (denoted $\mathrm{Suite}_1$), which we found insufficient for scalability testing. To ensure a rigorous evaluation, we extended $\mathrm{Suite}_1$ as follows:
\begin{itemize}
	\item $\mathrm{Suite}_2$ (399 benchmarks): $\mathrm{Suite}_1$ is from the benchmarks that were used to test a well-known model counter called Ganak \footnote{The benchmarks are available at \url{https://github.com/meelgroup/ganak}}; thereby, we added each circuit formula in the aforementioned domain but not in $\mathrm{Suite}_1$ from the Ganak benchmarks.
	\item $\mathrm{Suite}_3$ (459 benchmarks): Incorporated 60 additional combinatorial circuits \footnote{The additional benchmarks are available at \url{https://github.com/nianzelee/PhD-Dissertation}} from~\cite{lee2018solving}   on the basis of $\mathrm{Suite}_2$.
\end{itemize}
All experiments were run on a computer with Intel(R) Core(TM) i9-10920X CPU @ 3.50GHz and 32GB RAM.
Each instance was run on a single core with a timeout of 3000 seconds and 4GB memory, the same setup adopted in the evaluation of EntropyEstimation.

Through our experiments, we sought to answer the following research questions:
\begin{enumerate}[RQ1:]
	\item How does the runtime performance of PSE compare to the state-of-the-art Shannon entropy tools with (probabilistic) accuracy guarantee?
	\item How do the utilized methods impact the runtime performance of PSE? 
\end{enumerate}

\subsection{RQ1: Performance of PSE}

Golia et al.~\cite{golia2022scalable} have already demonstrated that their probably approximately correct tool EntropyEstimation is significantly more efficient than the state-of-the-art precise Shannon entropy tools. 
The comparative experiments between PSE and the state-of-the-art precise tools are presented in the appendix.
We remark that PSE significantly outperforms the precise baseline (the baseline was able to solve only 18 benchmarks, whereas PSE solved 332 benchmarks). 
This marked improvement is attributed to the linear entropy computation capability of \ADDAND and the effectiveness of various strategies employed in PSE.

Table \ref{table:benchmarks} presents the performance comparison between PSE and EntropyEstimation across the three benchmark suites.  
For $\mathrm{Suite}_1$, EntropyEstimation solved two more instances than PSE, indicating a slight advantage. 
However, among the 94 instances that both solved, PSE demonstrated higher efficiency.
Moreover, PSE achieved a lower PAR-2~\footnote{The PAR-2 scoring scheme gives a penalized average runtime, assigning a runtime of two times the time limit for each benchmark that the tool fails to solve} score than EntropyEstimation, suggesting that PSE holds an overall performance advantage.  
We remark that in the computation of the PAR-2 scores, we \emph{did not perform additional penalization} for each successful run of EntropyEstimation as its output was very close to the true entropy.
For $\mathrm{Suite}_2$, PSE solved 44 more instances than EntropyEstimation and achieved a significantly lower PAR-2 score, further demonstrating its superior performance.  
For $\mathrm{Suite}_3$, PSE solved 56 more instances than EntropyEstimation. Additionally, in terms of overall performance, PSE achieved a significantly lower PAR-2 score than EntropyEstimation, reinforcing its advantage. 
 
\begin{table}[!ht]
	\centering
	\renewcommand{\arraystretch}{1.2}
	\small
	\begin{tabular}{cc *{3}{@{\hspace{9pt}}c} c}
		\toprule
		\multirow{2}{*}{Suit}
		& \multirow{2}{*}{Tool}  
		& \multicolumn{3}{c}{Solved Instances} 
		& \multirow{2}{*}{PAR-2 score} \\
		\cline{3-5}
		& & Unique & Fastest & Total \\ 
		\midrule
		\multirow{2}{*}{$\mathrm{Suite}_1$} & EntropyEstimation & \textbf{2} & 0 & \textbf{96} & 149.63 \\
		{} & PSE (ours) & 0 & \textbf{94} & 94 & \textbf{126.04} \\\hdashline
		\multirow{2}{*}{$\mathrm{Suite}_2$} & EntropyEstimation & 2 & 0 & 268 & 2132.52 \\
		{} & PSE (ours) & \textbf{46} & \textbf{266} & \textbf{312} & \textbf{1314.59} \\\hdashline
		\multirow{2}{*}{$\mathrm{Suite}_3$} & EntropyEstimation & 2 & 0 & 276 & 2536.53 \\
		{} & PSE (ours) & \textbf{58} & \textbf{274} & \textbf{332} & \textbf{1666.11} \\

		\bottomrule
	\end{tabular}
	\caption{Detailed performance comparison of PSE and EntropyEstimation. Unique represents the number of instances that can only be solved by a specific tool. Fastest represents the number of instances that a tool solves with the shortest time.}
	\label{table:benchmarks}
\end{table}

Figure \ref{figure:scatter} demonstrates the detailed performance comparison between PSE and EntropyEstimation on $\mathrm{Suite}_3$.
More intuitively,  among all the benchmarks that both PSE and EntropyEstimation are capable of solving, in 98\% of those benchmarks, the efficiency of PSE surpasses that of EntropyEstimation by a margin of at least ten times.
For all the benchmarks where PSE and EntropyEstimation did not timeout and took more than 0.1 seconds, the mean speedup is 506.62, which indicates an improvement of more than two orders of magnitude. 

The aforementioned results clearly indicate that PSE outperforms EntropyEstimation in the majority of instances. This validates a positive answer to \textbf{RQ1}: PSE outperforms the state-of-the-art Shannon entropy tools with (probabilistic) accuracy guarantee.
We remark that EntropyEstimation is an estimation tool for Shannon entropy with probabilistic approximately correct results \cite{golia2022scalable}.
PSE consistently performs better than a state-of-the-art entropy estimator across most instances, highlighting that our methods significantly enhance the scalability of precise Shannon entropy computation.

\begin{figure}[h]
	\centering
	\includegraphics[width=0.6\linewidth, 
	height=0.8\textheight, 
	keepaspectratio]{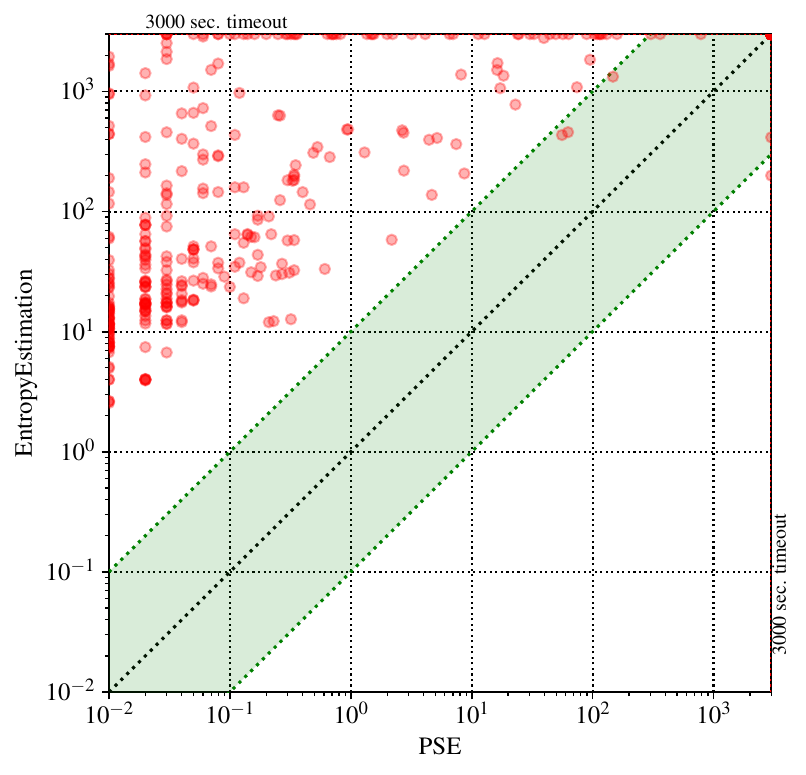}
	\caption{Scatter Plot of the running time Comparison between PSE and EntropyEstimation.
	}
	\label{figure:scatter}
\end{figure}

\subsection{RQ2: Impact of algorithmic configurations}
To better verify the effectiveness of the PSE methods and answer \textbf{RQ2}, we conducted a comparative study on all the utilized methods, including methods for the $Y$-stage: \textsf{Conjunctive Decomposition}, \textsf{YCache}, \textsf{Pre}, variable decision heuristics (\textsf{minfill}, \textsf{DLCP}, \textsf{SharpSAT-TD heuristic}, \textsf{VSADS}), and methods for the $X$-stage: \textsf{XCache} and \textsf{ConditionedCounting}.
In accordance with the principle of control variables, we conducted ablation experiments to evaluate the effectiveness of each method, ensuring that each experiment differed from the PSE tool by only one method.
The cactus plot for the different methods is shown in Figure \ref{figure:3}, where PSE represents our tool. 
PSE-wo-Decomposition indicates that the \textsf{ConjunctiveDecomposition} method is disabled in PSE, which means that its corresponding trace is ADD.
PSE-wo-Pre means that \textsf{Pre} is turned off in PSE.
PSE-ConditionedCounting indicates that PSE employed the \textsf{ConditionedCounting} method rather than \textsf{SharedCounting} in the $X$-stage.
PSE-wo-XCache indicates that the caching method is turned off in PSE in the $X$-stage. 
PSE-wo-YCache indicates that the caching method is turned off in PSE in the $Y$-stage.
PSE-dynamic-SharpSAT-TD means that PSE replaces the \textsf{minfill} static variable order with the dynamic variable order: the variable decision-making heuristic method of SharpSAT-TD (all other configurations remain identical to PSE, with only the variable heuristic differing).
Similarly, PSE-dynamic-DLCP and PSE-dynamic-VSADS respectively indicate the selection of dynamic heuristic \textsf{DLCP} and \textsf{VSADS}.

\begin{figure}[h]
	\centering
	\includegraphics[width=0.6\linewidth, 
	height=0.8\textheight, 
	keepaspectratio]{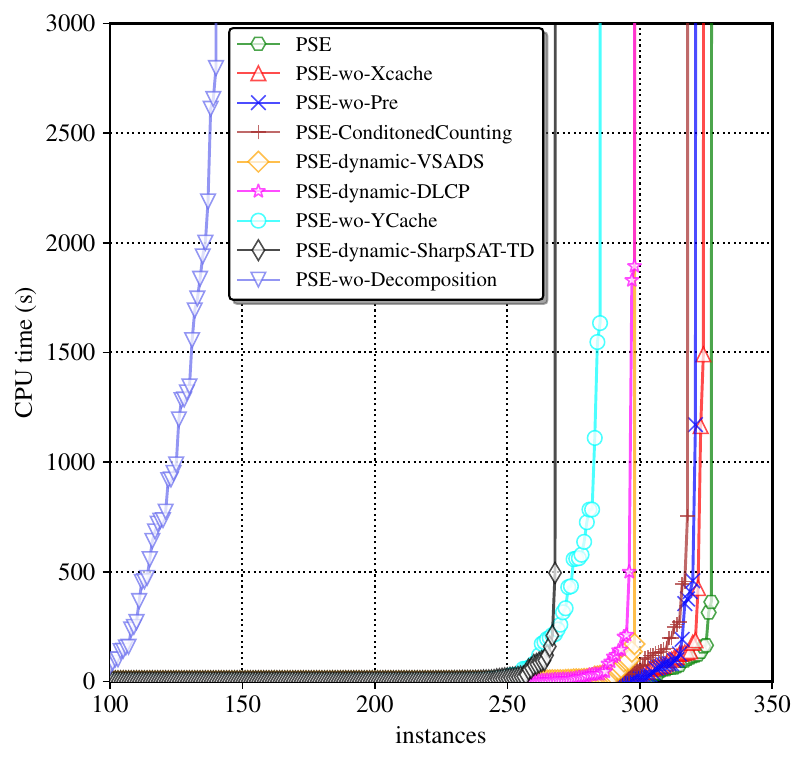}
	\caption{Cactus plot comparing different methods.}
	\label{figure:3}
\end{figure}

The experimental results highlight the significant effects of conjunctive decomposition.
Caching also demonstrates significant benefits, consistent with findings from previous studies on knowledge compilation. 
It can also be clearly observed that \textsf{Pre} improves the efficiency of PSE.
Among the heuristic strategies, it is evident that \textsf{minfill} performs the best.
In the technique of the $X$-stage, the \textsf{ConditionedCounting} method performs better than \textsf{SharedCounting} without \textsf{XCache}, but not as well as the \textsf{SharedCounting} method.
This comparative experiment indicates that the shared component caching is quite effective.
The \textsf{ConditionedCounting} method's major advantage is its linear time complexity~\cite{lai2017new}. 
However, a notable drawback is the requirement to construct an \OBDDAND (or other knowledge compilation languages such as d-DNNF, SDD, etc.) based on a static variable order, which can introduce considerable time overhead for more complex problems. 
Although the \textsf{ConditionedCounting} method is not the most effective, we believe it is still a promising and scalable method. 
In cases where an \ADDAND can be efficiently constructed based on a static variable order, the \textsf{ConditionedCounting} method may be more effective than the \textsf{SharedCounting} method, especially when modeling counting in the $X$-stage is particularly challenging.
Finally, PSE utilizes the \textsf{SharedCounting} strategy in the $X$-stage, and incorporates  \textsf{ConjunctiveDecomposition}, \textsf{YCache}, \textsf{Pre}, and the \textsf{minfill} heuristic method in the $Y$-stage.

Finally, we analyze the effectiveness of algorithmic configurations across benchmarks. 
In terms of the number of solved instances, PSE either solves the most instances or ties with other configurations across all domains. Regarding the PAR-2 score, on QBF benchmarks, PSE-dynamic-VSADS has the lowest score, while in other domains, PSE has the lowest scores.
Among all the instances, there are a total of two instances~\footnote{The two instances are  in bit-blasted versions of SMTLIB benchmarks with names blasted\_case\_0\_ptb\_1 and blasted\_TR\_b12\_1\_linear.} which PSE failed to solve within the specified time limit, but were solved by PSE-wo-Pre.
In PSE, we use the minfill heuristic to construct a tree decomposition for a given circuit formula. We also observed that the resulting treewidth strongly correlates with compilation size---smaller treewidth in a benchmark typically leads to more efficient PSE execution.

\section{Related work}
\label{sec:Related}
Our work is based on the close relationship between QIF, model counting, and knowledge compilation.
We introduce relevant work from three perspectives: (1) quantitative information flow analysis, (2) model counting, and (3) knowledge compilation.

\textbf{Quantified information flow analysis} 
At present, the QIF method based on model counting encounters two significant challenges.
The first challenge involves constructing the logical postcondition $\Pi_{proc}$ for a program $proc$~\cite{zhou2018static}.
Although symbolic execution can achieve this, existing symbolic execution tools have limitations and are often challenging to extend to more complex programs, such as those involving symbolic pointers.
The second challenge concerns model counting, a key focus of our research. 
For programs modeled by Boolean clause constraints, Shannon entropy can be computed via model counting queries, enabling the quantification of information leakage.
Golia et al.~\cite{golia2022scalable} have made notable contributions to this field. 
They proposed the first efficient Shannon entropy estimation method with PAC guarantees, utilizing sampling and model counting. 
Their approach focuses on reducing the number of model counting queries by employing sampling techniques. 
Nevertheless, this method yields only an approximate estimation of entropy.
Our research is motivated by the work of Golia et al., but diverges in its approach and optimization strategy. 
We enhance the existing model counting framework for precise Shannon entropy by reducing the number of model counting queries and concurrently improving the efficiency of model counting solutions.
Inspired by Golia et al.'s work, our research differs in approach and optimization strategy. 
We improve the existing model counting framework for precise Shannon entropy by reducing the number of model counting queries and enhancing solution efficiency. 

\textbf{Model counting}
Since the computation of entropy relies on model counting, we reviewed advanced techniques in this domain.
The most effective methods for exact model counting include component decomposition, caching, variable decision heuristics, pre-processing, and so on.
In our research, these methods can all be optimized and improved for application in Shannon entropy computation.
The fundamental principle of disjoint component analysis involves partitioning the constraint graph into separate components that do not share variables.
The core of \ADDAND lies in leveraging component decomposition to enhance the efficiency of construction.
We also utilized caching techniques in the process of computing entropy, and our experiments once again demonstrated the power of caching techniques.
Extensive research has been conducted on variable decision heuristics for model counting, which are generally classified into static and dynamic heuristics.
In static heuristics, the \textsf{minfill}~\cite{darwiche2009modeling} heuristic is notably effective, while in dynamic heuristics, \textsf{VSADS}~\cite{sang2005heuristics}, \textsf{DLCP}~\cite{lai2021power}, and \textsf{SharpSAT-TD  heuristic}~\cite{korhonen2021integrating} have emerged as the most significant in recent years.
Lagniez et al.~\cite{lagniez2017preprocessing} offer a comprehensive review of preprocessing techniques in model counting.

\textbf{Knowledge compilation} 
The motivation for knowledge compilation lies in transforming the original representation into a target language to enable efficient solving of inference tasks.
Darwiche et al. first proposed a compiler called c2d~\cite{darwiche2004new} to convert the given CNF formula into Decision-DNNF. 
Lai et al. proposed two extended forms of OBDD:
Ordered Binary Decision Diagram with Implied Literals (OBDD-L~\cite{lai2013reduced}), which is developed by extracting implied literals recursively;
\OBDDAND~\cite{lai2017new}, which is proposed by integrating conjunctive decomposition.
Both forms aim to reduce the size of OBDD.
Exploiting literal equivalence, Lai et al.~\cite{lai2021power} proposed a generalization of Decision-DNNF, called CCDD, to capture literal equivalence. 
They demonstrate that CCDD supports model counting in linear time and design a model counter called ExactMC based on CCDD.
In order to compute the Shannon entropy, the focus of this paper is to design a compiled language that supports the representation of probability distributions.
Numerous target representations have been used to concisely model probability distributions. 
For example, d-DNNF can be used to compile relational Bayesian networks for exact inference~\cite{chavira2006compiling};
Probabilistic Decision Graph (PDG) is a representation language for probability distributions based on BDD~\cite{jaeger2004probabilistic}.
Macii and Poncino~\cite{macii1996exact} utilized knowledge compilation to calculate entropy, demonstrating that ADD enables efficient and precise computation of entropy.
However, the size of ADD often grows exponentially for large scale circuit formulas. 
To simplify ADD size, we propose an extended form, \ADDAND. It uses conjunctive decomposition to streamline the graph structure and facilitate cache hits during construction.

\section{Conclusion}
\label{sec:Conclusion}

In this paper, we propose a new compilation language, \ADDAND, which combines ADD and conjunctive decomposition to optimize the search process in the first stage of precise Shannon entropy computation. 
In the second stage of precise Shannon entropy computation, we optimize model counting queries by utilizing the shared component cache.
We integrated preprocessing, heuristics, and other methods into the precise Shannon computation tool PSE, with its trace corresponding to \ADDAND. 
Experimental results demonstrate that PSE significantly enhances the scalability of precise Shannon entropy computation, even outperforming the state-of-the-art entropy estimator EntropyEstimation in overall performance.
We believe that PSE has opened up new research directions for entropy computing in Boolean formula modeling.

\appendix

\section{Comparison with precise Shannon entropy computing methods}
\label{sec:Appendix}

In the appendix, we compare PSE with the state-of-the-art precise methods of computing Shannon entropy.
The existing precise Shannon entropy tools do not use the techniques in the state-of-the-art model counters.
Just like~\cite{golia2022scalable}, we implemented the precise Shannon entropy baseline with state-of-the-art model counting techniques.
In the baseline, we enumerate each assignment $\sigma \in \mathit{Sol}(\varphi)_{\downarrow Y}$ and compute $p_{\sigma} = \frac{\left| \mathit{Sol}(\varphi(Y \mapsto \sigma)) \right|}{ \left| \mathit{Sol}(\varphi)_{\downarrow X} \right| }$, where $\mathit{Sol}(\varphi(Y \mapsto \sigma))$ denotes the set of solutions of $\varphi(Y \mapsto \sigma)$ and $\mathit{Sol}(\varphi)_{\downarrow X}$ denotes the set of solutions of $\varphi$ projected to $X$.
As can be seen from the previous proposition, $| \mathit{Sol}(\varphi)_{\downarrow X} |$ can be replaced by $| \mathit{Sol}(\varphi) |$.
Finally, entropy is computed as $H(\varphi) = \sum_{\sigma \in 2^Y} -p_{\sigma} \log {p_{\sigma}} $.
For a formula with an output set size of $m$, $2^m$ model counting queries are required. 
For model counting queries, we have adopted two different methods. 
One is to directly invoke the currently state-of-the-art model counters, and our experiment, SharpSAT-TD, Ganak, and ExactMC are employed. 
The other method involves utilizing knowledge compilation. Firstly, we construct an offline knowledge compilation language that supports linear model counting, and then perform online conditioning based on each assignment over the $Y$ variables.
The knowledge compilation language in our experiment is (\OBDDAND via KCBox), and this method corresponds to the baseline-Panini in Table \ref{table:precise}.
Panini is an efficient compilation tool that supports the compilation of CNF formulas into the form of \OBDDAND to enable efficient model counting. 

\begin{table}[!h]
	\centering
	\resizebox{\linewidth}{!}
	{
		\footnotesize 
		\setlength{\tabcolsep}{3pt} 
		\begin{tabular}{ crrrrrrrrrrrrr } 
			\toprule 
			\multirow{2}*{instance} & \multirow{2}*{$\left| X \right|$} & \multirow{2}*{$\left| Y \right|$} & \multicolumn{2}{c}{baseline-SharpSAT-TD} & \multicolumn{2}{c}{baseline-ExactMC} &
			\multicolumn{2}{c}{baseline-Ganak} & \multicolumn{2}{c}{baseline-Panini} & \multicolumn{2}{c}{PSE} \\
			
			\cmidrule(r){4-5}   
			\cmidrule(r){6-7}
			\cmidrule(r){8-9}  
			\cmidrule(r){10-11}
			\cmidrule(r){12-13}
			
			& & & Entropy & Time(s) & Entropy & Time(s) & Entropy & Time(s)  & Entropy & Time(s)  & Entropy & Time(s)\\ 
			
			\midrule  
			blasted$\_$case102.cnf & 11 & 23 & 8 & 81.29 & 8 & 0.39 & 8 & 35.42 & 8 & 1.82 & 8 & 0.16 \\
			s27$\_$15$\_$7.cnf & 7 & 25 & 3.3 & 142.77  & 3.3 & 0.15 & 3.3 & 0.25 & 3.3 & 0.31 & 3.3 & 0.14 \\
			small-bug1-fixpoint-5.cnf & 66 & 21 & 12.81 & 79.51 & 12.81 & 2.99 & 12.81 & 104.17 & 12.81 & 8.79 & 12.81 & 0.18 \\
			small-bug1-fixpoint-6.cnf & 79 & 25 & 15.31 & 1406.47 & 15.31 & 51.9 & 15.31 & 1846.36 & 15.31 & 127.70 & 15.31  & 0.21 \\
			blasted$\_$case144.cnf & 138 & 627 & - & - & - & - & - & - & - & - & 77.62 & 35.42 \\
			s1423a$\_$15$\_$7.cnf & 91 & 773 & - & - & - & - & - & - & - & - &  88.17 & 232.65 \\
			s382$\_$15$\_$7.cnf & 24 & 326 & - & - & - & - & - & - & - & - & 23.58 & 0.22 \\
			CVE-2007-2875.cnf & 752 & 32 & - & - & - & - & - & - & - & - & 32 & 0.72 \\
			10.sk$\_$1$\_$46.cnf & 47 & 1447 & - & - & - & - & - & - & - & - &  13.58 & 0.18 \\
			
			\bottomrule
		\end{tabular}
		
	}
	\caption{Entropy computation performance of baselines and PSE. "-" represents that the entropy cannot be computed within the specified time limit.
		\label{table:precise}
	}
\end{table}

Our experimental results indicate that all four representative state-of-the-art exact Shannon entropy baselines can only solve 18 benchmarks within the time limit of 3000 seconds, whereas PSE can solve 332 benchmarks. 
Table \ref{table:precise} shows the comparison between baselines and PSE on some instances.
Notably, although some instances have similar sizes of $X$ and $Y$ sets, their computation times vary significantly (e.g., blasted\_case144.cnf vs. s1423a\_15\_7.cnf). To clarify, computation times depend on multiple parameters, such as an exponential relationship with treewidth in addition to problem size. We employ the minfill heuristic to compute tree decompositions, guiding the entropy calculation. Our experimental results show that blasted\_case144.cnf has a minfill treewidth of 22, whereas s1423a\_15\_7.cnf has a minfill treewidth of 27.
The results show a significant improvement in the efficiency of PSE for computing the precise Shannon entropy.
We remark that the poorer performance of these baselines is due to the exponential size of $2^Y$.

\bibliography{references}

\appendix

\end{document}